\documentclass[12pt]{colt2018}

\usepackage{microtype}
\usepackage{graphicx}
\usepackage{booktabs}

\usepackage{hyperref}

\usepackage{graphicx}
\usepackage[utf8]{inputenc} % allow utf-8 input
\usepackage[T1]{fontenc}    % use 8-bit T1 fonts
\usepackage{hyperref}       % hyperlinks
\usepackage{url}            % simple URL typesetting
\usepackage{booktabs}       % professional-quality tables
\usepackage{amsfonts}       % blackboard math symbols
\usepackage{nicefrac}       % compact symbols for 1/2, etc.
\usepackage{microtype}      % microtypography
\usepackage{url}            % simple URL typesetting
\usepackage{booktabs}       % professional-quality tables
\usepackage{amsfonts}       % blackboard math symbols
\usepackage{nicefrac}       % compact symbols for 1/2, etc.
\usepackage{microtype}      % microtypography
\usepackage{nccmath}
\usepackage{times}
\usepackage{graphicx}
\usepackage{natbib}
\usepackage{algorithm}
\usepackage{algorithmic}
\usepackage{hyperref}
\usepackage{amsfonts}
\usepackage{amsmath}
\usepackage{mathtools}
\usepackage{graphicx}
\usepackage{enumerate}
\usepackage{algorithm}
\usepackage{algorithmic}
\usepackage{thm-restate}

\usepackage{url}
\usepackage{dsfont}
\usepackage[mathscr]{euscript}
\usepackage{booktabs}
\usepackage{makecell}
\usepackage{xcolor, colortbl}
\usepackage[normalem]{ulem}
\usepackage{soul}
\usepackage{prettyref}

\definecolor{Gray}{gray}{0.85}
\newcolumntype{g}{>{\columncolor{Gray}}c}
\hypersetup{
  colorlinks   = true,
  urlcolor     = blue,
  linkcolor    = blue,
  citecolor   = blue
}

\def\Rset{\mathbb{R}}

\DeclareMathOperator*{\E}{\mathbb{E}}

\DeclareMathOperator*{\argmin}{\rm argmin}

\DeclarePairedDelimiter{\abs}{\lvert}{\rvert} 
\DeclarePairedDelimiter{\bracket}{[}{]}

\newcommand{\cD}{\mathcal{D}}
\newcommand{\cL}{\mathcal{L}}

\newcommand{\sD}{{\mathscr D}}

\newcommand{\sP}{{\mathscr P}}
\newcommand{\sQ}{{\mathscr Q}}
\newcommand{\sR}{{\mathscr R}}

\newcommand{\sX}{{\mathscr X}}
\newcommand{\sY}{{\mathscr Y}}

\newcommand{\sfD}{{\mathsf D}}
\newcommand{\sfd}{{\mathsf d}}
\newcommand{\sfp}{\mspace{1mu} {\mathsf p}}

\newcommand{\dmsa}{\texttt{\small DMSA}}
\newcommand{\gmsa}{\texttt{\small GMSA}}
\newcommand{\unif}{\texttt{\small unif}}
\newcommand{\joint}{\texttt{\small joint}}

\newcolumntype{P}[1]{>{\centering\arraybackslash}p{#1}}

\newrefformat{th}{Theorem~\ref{#1}}
\newrefformat{prop}{Proposition~\ref{#1}}
\newrefformat{lemma}{Lemma~\ref{#1}}

\newcommand{\h}{\widehat}

\newcommand{\e}{\epsilon}

\renewcommand{\set}[2][]{#1 \{ #2 #1 \} }
\newcommand{\ignore}[1]{}

\title{A Discriminative Technique for Multiple-Source Adaptation}

\coltauthor{\Name{Corinna Cortes}\Email{corinna@google.com}\\
\addr Google Research 
\AND
\Name{Mehryar Mohri}\Email{mohri@google.com}\\
\addr Google Research and 
Courant Institute of Mathematical Sciences, New York
\AND
\Name{Ananda Theertha Suresh}\Email{theertha@google.com}\\
\addr Google Research, New York
\AND
\Name{Ningshan Zhan}\Email{nzhang@stern.nyu.edu}\\
\addr Hudson River Trading
}

\begin{document}
\maketitle
\begin{abstract}

  We present a new discriminative technique for the multiple-source
  adaptation, MSA, problem. Unlike previous work, which relies on density
  estimation for each source domain, our solution only requires
  conditional probabilities that can easily be accurately estimated from
  unlabeled data from the source domains. We give a detailed analysis
  of our new technique, including general guarantees based on R\'enyi
  divergences, and learning bounds when conditional Maxent is used for
  estimating conditional probabilities for a point to belong to a source domain. We show that these guarantees
  compare favorably to those that can be derived for the generative
  solution, using kernel density estimation.  Our experiments with
  real-world applications further demonstrate that our new
  discriminative MSA algorithm outperforms the previous generative
  solution as well as other domain adaptation baselines.

\end{abstract}

\section{Introduction}

Learning algorithms are applied to an increasingly broad array of
problems. For some tasks, large amounts of labeled data are available
to train very accurate predictors.  But, for most
new problems or domains, no such supervised information is at the
learner's disposal. Furthermore, labeling data is costly since it
typically requires human inspection and agreements between multiple
expert labelers. Can we leverage past predictors learned for various
domains and combine them to devise an accurate one for a new task?
Can we provide guarantees for such combined predictors? How should we
define that combined predictor?  These are some of the challenges of
\emph{multiple-source domain adaptation}.

The problem of domain adaptation from multiple sources admits
distinct instances defined by the type of source information
available to the learner, the number of source domains, and the
amount of labeled and unlabeled data available from the target domain
\citep{MansourMohriRostamizadeh2008,MansourMohriRostamizadeh2009,
hoffman2018algorithms,pan_tkda2010,MuandetBalduzziScholkopf2013,xu_eccv14,
hoffman_eccv12,gong_icml13,gong_nips13,zhang2015multi,
ganin2016domain, tzeng2015simultaneous, motiian2017unified,
motiian2017few, wang2019few, konstantinov2019robust, liu2015multiple, saito2019semi, wang2019transfer}.
The specific instance we are considering is one where the learner has
access to multiple source domains and where, for each
domain, they only have at their disposal a predictor trained for that domain 
and some amount of unlabeled data. No other information 
about the source domains, in particular no labeled data
is available. The target domain or distribution is unknown but it
is assumed to be in the convex hull of the source distributions,
or relatively close to that.
The \emph{multiple-source adaptation (MSA) problem} consists of
combining relatively accurate predictors available for each source
domain to derive an accurate predictor for \emph{any} such new mixture
target domain. This problem was first theoretically studied by
\citet{MansourMohriRostamizadeh2008,MansourMohriRostamizadeh2009} and
subsequently by \citet{hoffman2018algorithms}, who further provided an
efficient algorithm for this problem and reported the results of a series of
experiments with that algorithm and favorable comparisons with 
alternative solutions.

As pointed out by these authors, this problem arises in a variety of
different contexts. In speech recognition, each domain may correspond
to a different group of speakers and an acoustic model learned for
each domain may be available. Here, the problem consists of devising a
general recognizer for a broader population, a mixture of the source
domains \citep{liao_icassp13}. Similarly, in object recognition, there
may be accurate models trained on different image databases and the
goal is to come up with an accurate predictor for a general domain,
which is likely to be close to a mixture of these sources
\citep{efros_cvpr11}. A similar situation often appears in sentiment
analysis and various other natural language processing problems where
accurate predictors are available for some source domains such as TVs,
laptops and CD players, each previously trained on labeled data, but
no labeled data or predictor is at hand for the broader category
of electronics, which can be viewed as a mixture of the sub-domains
\citep{blitzer_acl07, dredze_nips08}.

An additional motivation for this setting of multiple-source
adaptation is that often the learner does not have access to labeled
data from various domains for legitimate reasons such as privacy or
storage limitation. This may be for example labeled data from various
hospitals, each obeying strict regulations and privacy rules. But, a
predictor trained on the labeled data from each hospital may be
available. Similarly, a speech recognition system trained on data from
some group may be available but the many hours of source labeled data
used to train that model may not be accessible anymore, due to the
very large amount of disk space it requires.  Thus, in many cases, the
learner cannot simply merge all source labeled data to learn a
predictor.

\textbf{Main contributions.}  In Section~\ref{sec:solutions}, we
present a new \emph{discriminative technique} for the MSA problem,
Previous work showed that a \emph{distribution-weighted combination}
of source predictors benefited from favorable theoretical guarantees
\citep{MansourMohriRostamizadeh2008,
  MansourMohriRostamizadeh2009,hoffman2018algorithms}.  However, that
\emph{generative solution} requires an accurate density estimation for
each source domain, which, in general, is a difficult problem.
Instead, our solution only needs conditional probabilities, which is easier to
accurately estimate from unlabeled data from the source domains.
We also describe an efficient DC-programming optimization algorithm
for determining the solution of our discriminative technique, which is
somewhat similar to but distinct from that of previous work, since it
requires a new DC-decomposition.

In Section~\ref{sec:theory}, we give a new and detailed theoretical
analysis of our technique, starting with new general guarantees that
depend on the R\'enyi divergences between the target distribution and
mixtures of the true source distributions, instead of mixtures of
estimates of those distributions (Section~\ref{sec:solutions}).  We
then present finite sample learning bounds for our new discriminative
solution when conditional Maxent is used for estimating conditional
probabilities.  We also give a new and careful analysis of the
previous generative solution, when using kernel density estimation,
including the first finite sample generalization bound for that
technique.
We show that the theoretical guarantees for our discriminative
solution compare favorably to those derived for the generative
solution in several ways.
While we benefit from some of the analysis in previous work
\citep{hoffman2018algorithms}, our proofs and techniques
for both solutions are new and non-trivial.

We further report the results of several experiments with our
discriminative algorithm both with a synthetic dataset and several
real-world applications (Section~\ref{sec:experiments}). Our results
demonstrate that, in all tasks, our new solution outperforms the
previous work's generative solution, which had been shown itself to
surpass empirically the accuracy of other domain adaptation baselines
\cite{hoffman2018algorithms}.  They also indicate that our
discriminative technique requires fewer samples to achieve a high
accuracy than the previous solution, which matches our theoretical
analysis.

\textbf{Related work.} There is a very broad literature dealing with
single-source and multiple-source adaptation with distinct scenarios.
Here, we briefly discuss the most related previous work, in addition
to \citep{MansourMohriRostamizadeh2008,
  MansourMohriRostamizadeh2009,hoffman2018algorithms}.
The idea of using a domain classifier to combine domain-specific
predictors has been suggested in the past.  \citet{jacobs1991adaptive}
and \citet{nowlan1991evaluation} considered an adaptive mixture of
experts model, where there are multiple expert networks, as well as a
gating network to determine which expert to use for each input. The
learning method consists of jointly training the individual expert
networks and the gating network.  In our scenario, no labeled data is
available, expert networks are pre-trained separately from the gating
network, and our gating network admits a specific structure.
\citet{hoffman_eccv12} learned a domain classifier via SVM on all
source data combined, and predicted on new test points with the
weighted sum of domain classifier's scores and domain-specific
predictors. Such linear combinations were later shown by
\citet{hoffman2018algorithms} to perform poorly in some cases and not
to benefit from strong guarantees.  More recently, \citet{xu2018deep}
deployed multi-way adversarial training to multiple source domains to
obtain a domain discriminator, and also used a weighted sum of
discriminator's scores and domain-specific predictors to make
predictions.  \citet{zhao2018adversarial} considered a scenario where
labeled samples are available, unlike our scenario, and learned a
domain classifier to approximate the discrepancy term in a MSA
generalization bound, and proposed the MDAN model to minimize the
bound. 

We start with a description of the learning scenario we consider and
the introduction of notation and definitions relevant to our analysis
(Section~\ref{sec:setup}).

%\newpage
\section{Learning Scenario}
\label{sec:setup}

We consider the MSA problem in the
general stochastic scenario studied by \citet{hoffman2018algorithms}
and adopt the same notation.

Let $\sX$ denote the input space, $\sY$ the output space.  We will
identify a \emph{domain} with a distribution over $\sX \times
\sY$. There are $p$ source domains $\sD_1, \ldots, \sD_p$. 
As in previous work, we adopt the assumption that the domains share a
common conditional probability $\sD(\cdot | x)$ and thus
$\sD_k(x, y) = \sD_k(x) \sD(y | x)$, for all
$(x, y) \in \sX \times \sY$ and $k \in [p]$.  This is a natural
assumption in many common machine learning tasks. For example, in
image classification, the label of a picture as a \emph{dog} may not
depend much on whether the picture is from a personal collection or a
more general dataset.  Nevertheless, as discussed in \citet{hoffman2018algorithms}, this condition can be relaxed and, here
too, all our results can be similarly extended to a more general case
where the conditional probabilities vary across domains. Since not all $k$ conditional probabilities are equally accurate on the single $x$, better target accuracy can be obtained by combining the $\sD_k(x)$s in an $x$-dependent way.

For each domain $\sD_k$, $k \in [p]$, the learner has access to some
unlabeled data drawn i.i.d.\ from the marginal distribution $\sD_k$
over $\sX$, as well as to a predictor $h_k$.  We consider two types of
predictor functions $h_k$, and their associated loss functions $\ell$
under the \emph{regression model (R)} and the \emph{probability model
  (P)} respectively:
\begin{align*}
  & h_k \colon \sX \to \Rset
  && \ell \colon \Rset \times \sY \to \Rset_+ & \text{(\emph{R})} \\
  & h_k  \colon \sX \times \sY \to [0, 1]
  && \ell \colon [0, 1] \to \Rset_+ & \text{(\emph{P})}
\end{align*}
In the probability model, the predictors are assumed to be 
normalized: $\sum_{y \in \sY} h(x, y) = 1$ for all $x \in \sX$.
We will denote by
$\cL(\sD, h)$ the expected loss of a predictor $h$ with respect to the
distribution $\sD$:
\begin{align*}
  \cL(\sD, h) & = \E_{(x, y) \sim \sD} \big[\ell( h(x), y) \big] 
\quad\text{(\emph{R})}, \\
  \cL(\sD, h) & = \E_{(x, y) \sim \sD} \big[\ell(h(x, y)) \big] 
\quad\text{(\emph{P})}.
\end{align*}
Our theoretical results are general and only assume that the loss
function $\ell$ is convex, continuous.  But, in the regression model,
we will be particularly interested in the squared loss
$\ell(h(x), y) = (h(x) - y)^2$ and, in the probability model, the
cross-entropy loss (or $\log$-loss) $\ell(h(x, y)) = -\log h(x, y)$.

We will also assume that each source predictor $h_k$ is $\e$-accurate
on its domain for some $\e > 0$, that is,
$\forall k \in [p], \cL(\sD_k, h_k) \leq \e$. Our assumption that the
loss of $h_k$ is bounded, implies that $\ell(h_k(x), y) \leq M$ or
$\ell(h_k(x,y))\leq M$, for all $(x, y) \in \sX \times \sY$ and
$k \in [p]$.

Let $\Delta = \set{\lambda = (\lambda_1, \ldots, \lambda_p) \colon \sum_{k
    = 1}^p \lambda_k = 1, \lambda_k \geq 0}$ denote the simplex in $\Rset^p$,
and
let $\cD = \set{\sD_\lambda \colon \sD_\lambda = \sum_{k = 1}^p \lambda_k
  \sD_k, \lambda \in \Delta}$ be the family of all mixtures of the source domains, that is
the convex hull of $\sD_k$s. 

Since not all $k$ source predictors are necessarily equally accurate on the single input $x$, better target accuracy can be obtained by combining the $h_k(x)$s dependent on $x$.
The MSA problem for the learner is exactly how to combine these source
predictors $h_k$ to design a predictor $h$ with small expected loss
for any unknown target domain $\sD_T$ that is an element of $\cD$, or any unknown distribution $\sD_T$
close to $\cD$.

Our theoretical guarantees are presented in terms of \emph{R\'enyi
  divergences}, a broad family of divergences between distributions
generalizing the relative entropy.  The R\'enyi Divergence is
parameterized by $\alpha \in [0, +\infty]$ and denoted by
$\sfD_\alpha$. The $\alpha$-R\'enyi Divergence between two
distributions $\sP$ and $\sQ$ is defined by:
\begin{equation*}
  \sfD_\alpha(\sP \parallel \sQ) = \frac{1}{\alpha - 1}
  \log \mspace{-2mu} \bracket[\Bigg]{ \sum_{(x, y) \in \sX \times \sY}
  \mspace{-20mu} \sP(x, y)  \left[ \frac{\sP(x, y)}{\sQ(x, y)}
  \right]^{\alpha - 1} }\mspace{-3mu},
\end{equation*}
where, for $\alpha \in \set{0, 1, +\infty}$, the expression is defined
by taking the limit \citep{arndt}. For $\alpha = 1$, the R\'enyi
divergence coincides with the relative entropy. We will denote by
$\sfd_\alpha(\sP \parallel \sQ)$ the exponential of
$\sfD_\alpha(\sP \parallel \sQ)$:
\begin{equation*}
\sfd_\alpha(\sP \parallel \sQ) 
= \bracket[\Bigg]{ \sum_{(x, y) \in \sX \times \sY} 
\frac{\sP^\alpha(x, y)}{\sQ^{\alpha - 1}(x, y)}
}^{\frac{1}{\alpha - 1}}. 
\end{equation*}

In the following, to alleviate the notation, we abusively denote the marginal
distribution of a distribution $\sD_k$ defined over $\sX \times \sY$
in the same way and rely on the arguments for disambiguation, e.g.
$\sD_k(x)$ vs. $\sD_k(x, y)$.

\section{Discriminative MSA solution}
\label{sec:solutions}

In this section we present our new solution for the MSA problem and give an efficient
algorithm for determining its parameter. But first we describe the previous solution.

\subsection{Previous Generative Technique}

In previous work, it was shown that, in general, standard convex
combinations of source predictors can perform poorly
\citep{MansourMohriRostamizadeh2008,MansourMohriRostamizadeh2009,hoffman2018algorithms}:
in some problems, even when the source predictors have zero loss, no
convex combination can achieve a loss below some constant for a
uniform mixture of the source distributions.  Instead, a
\emph{distribution-weighted} solution was proposed to the MSA
problem. That solution relies on density estimates $\h \sD_k$ for the
marginal distributions $x \mapsto \sD_k(x)$, which are obtained via
techniques such as kernel density estimation, for each source domain
$k \in [p]$ independently.

Given such estimates, the solution is defined as follows in the
regression and probability models, for all
$(x, y) \in \sX \times \sY$:
\begin{align}
\h h_z(x) & = \sum_{k = 1}^p \frac{z_k \h \sD_k(x) }
{\sum_{j = 1}^p z_j \h \sD_j(x) } h_k(x), 
\label{eq:h_reg} \\
\h h_z(x, y) & = \sum_{k = 1}^p\frac{z_k \h \sD_k(x) }{
\sum_{j = 1}^p z_j \h \sD_j(x) } h_k(x, y),
\label{eq:h_prob}
\end{align}
with $z \in \Delta$ is a parameter determined via an optimization
problem such that $h_z$ admits the same loss for all $\sD_k$. We are assuming here that the estimates verify
$\h \sD_k(x) > 0$ for all $x \in \sX$ and therefore that the
denominators are positive. Otherwise, a small positive number
$\eta > 0$ can be added to the denominators of the solutions, as in
previous work. We are adopting this assumption only to simplify
the presentation. For the probability model, the joint estimates
$\h \sD_k(x, y)$ used in \citep{hoffman2018algorithms} can be
equivalently replaced by marginal ones $\h \sD_k(x)$ since all domain
distributions share the same conditional probabilities.

Since this previous work relies on density estimation, we will refer to it
as a \emph{generative solution to the MSA problem}, in short,
\gmsa. The technique benefits from the following general guarantee
\citep{hoffman2018algorithms}, where we extend the
R\'enyi divergences to divergences between a 
distribution $\sD$ and a set of distributions $\cD$ and write $\sfD_\alpha(\sD \parallel \cD) = \min_{\sD \in \cD} \sfD_\alpha(\sD \parallel \sD)$.

\begin{theorem}
\label{th:nestimate}
For any $\delta > 0$, there exists a $z \in \Delta$ such that the
following inequality holds for any $\alpha > 1$ and arbitrary target
distribution $\sD_T$:
\begin{equation*}
\cL(\sD_T, \h h_z) 
\leq \left[(\h \e + \delta) \, \sfd_\alpha(\sD_T \parallel \h \cD)
\right]^{\frac{\alpha- 1}{\alpha}} M^{\frac{1}{\alpha}},
\end{equation*}
where $\h \e = \max_{k \in [p]} \Big[\e  \, \sfd_\alpha(\h \sD_k \parallel \sD_k)
\Big]^{\frac{\alpha - 1}{\alpha}} M^{\frac{1}{\alpha }}$, 
and $\h \cD = \left\{ \sum_{k = 1}^p \lambda_k \h \sD_k\colon \lambda \in \Delta \right\}$.
\end{theorem}
The bound depends on the quality of the density estimates via the
R\'enyi divergence between $\h \sD_k$ and $\sD_k$, for each
$k \in [p]$, and the closeness of the target distribution $\sD_T$ to
the mixture family $\h \cD$. For
$\alpha = +\infty$, for $\sD_T$ close to $\h \cD$ and accurate
estimates of $\sD_k$, $\sfd_\alpha(\sD_T \parallel \h \cD)$ and
$\sfd_\alpha(\h \sD_k \parallel \sD_k)$ are close to one and the upper
bound is as a result close to $\e$. That is, with good density estimates, the
error of $h_z$ is no worse than that of the source predictors
$h_k$s. However, obtaining good density estimators is a difficult
problem and in general requires large amounts of data. In the
following section, we provide a new and less data-demanding solution
based on conditional probabilities.

\subsection{New Discriminative Technique}
\label{subsec:gz}

Let $\sD$ denote the distribution over $\sX$ defined by
$\sD(x) = \frac{1}{p} \sum_{k = 1}^p \sD_k(x)$. We will assume and can enforce that $\sD$ is the
distribution according to which we can expect to receive unlabeled
samples from the $p$ sources to train our discriminator. 
We will denote by $\sQ$ the
distribution over $\sX \times [p]$ defined by
$\sQ(x, k) = \frac{1}{p} \sD_k(x)$, whose $\sX$-marginal
coincides with $\sD$: $\sQ(x) = \sD(x)$.

Our new solution relies on estimates $\h \sQ(k | x)$ of the
conditional probabilities $\sQ(k | x)$ for each domain $k \in [p]$, that is
the probability that point $x$ belongs to source $k$.  Given such
estimates, our new solution to the MSA problem is defined as follows
in the regression and probability models, for all
$(x, y) \in \sX \times \sY$:
\begin{align}
\h g_z(x) 
& = \sum_{k = 1}^p \frac{z_k \h \sQ(k | x) }
{\sum_{j = 1}^p z_j \h \sQ(j | x) } h_k(x),
\label{eq:g_reg} \\
\h g_z(x, y) & = \sum_{k = 1}^p\frac{z_k \h \sQ(k | x) }{
\sum_{j = 1}^p z_j \h \sQ(j | x) } h_k(x, y),
\label{eq:g_prob}
\end{align} 
with $z \in \Delta$ being a parameter determined via an optimization
problem. As for the \gmsa\ solution, we are assuming here that the
estimates verify $\h \sQ(k | x) > 0$ for all $x \in \sX$ and therefore
that the denominators are positive. Otherwise, a small positive number
$\eta > 0$ can be added to the denominators of the solutions, as in
previous work. We are adopting this assumption only to simplify
the presentation.  Note that in the probability model, $\h g_z(x, y)$
is normalized since $h_k$s are normalized:
$\sum_{y \in \sY} g_z(x, y) = 1$ for all $x \in \sX$.

Since our solution relies on estimates of conditional probabilities of domain membership,
we will refer to it as a \emph{discriminative solution to the MSA
  problem}, \dmsa\ in short.

Observe that, by the Bayes' formula, the conditional probability
estimates $\h \sQ(k | x)$ induce density estimates $\h \sD_k(x)$ of
the marginal distributions $x \mapsto \sD_k(x)$:
\begin{equation}
\label{eq:estimate_D_k}
\h \sD_k(x) = \frac{\h \sQ(k | x) \sD(x)}{\h \sQ(k)}
\end{equation}
where $\h \sQ(k) = \sum_{x \in \sX} \h \sQ(k | x) \sD(x)$.  For an
exact estimate, that is $\h \sQ(k | x) = \sQ(k | x)$, the formula
holds with $\h \sQ(k) = \sum_{x \in \sX} \sQ(x, k) = \frac{1}{p}$.  In
light of this observation, we can establish the following connection
between the \gmsa\ and \dmsa\ solutions.

\begin{proposition}
\label{prop:hzgz}
Let $\h h_z$ be the \gmsa\ solution using the estimates
$\h \sD_k$ defined in \eqref{eq:estimate_D_k}. Then, for any
$z \in \Delta$, we have $\h h_z = \h g_{z'}$ with 
$z'_k = \frac{z_k/\h \sQ(k)}{\sum_{j = 1}^p z_j/\h \sQ(j)}$,
for all $k \in [p]$.
\end{proposition}
\begin{proof}
First consider the regression model. By definition of the \gmsa\
solution, we can write:
\begin{align*}
\h h_z(x) 
& = \sum_{k = 1}^p \frac{z_k \frac{\h \sQ(k | x) \sD(x)}{\h
  \sQ(k)}}{\sum_{j = 1}^p z_j \frac{\h \sQ(j | x) \sD(x)}{\h \sQ(j)}} h_k(x) \\
& = \sum_{k = 1}^p \frac{\frac{z_k}{\h \sQ(k)} \h \sQ(k | x) }{\sum_{j
  = 1}^p \frac{z_j}{\h \sQ(j)} \h \sQ(j | x)} h_k(x) 
= g_{z'}(x).
\end{align*}
The probability model's proof is 
syntactically the same.
\end{proof}
In view of this result, the \dmsa\ technique benefits from a
guarantee similar to \gmsa\ (Theorem~\ref{th:nestimate}),
where for \dmsa\ the density estimates are based on the conditional probability
estimates $\h \sQ(k | x)$. We refer readers to the full version of the paper for all the proofs.
\begin{theorem}
\label{th:nestimate-bis}
For any $\delta > 0$, there exists a $z \in \Delta$ such that the
following inequality holds for any $\alpha > 1$ and arbitrary target
distribution $\sD_T$:
\begin{equation*}
\cL(\sD_T, \h g_z) 
\leq \left[(\h \e + \delta) \, \sfd_\alpha(\sD_T \parallel \h \cD)
\right]^{\frac{\alpha- 1}{\alpha}} M^{\frac{1}{\alpha}},
\end{equation*}
where $\h \e = \max_{k \in [p]} \Big[\e  \, \sfd_\alpha(\h \sD_k \parallel \sD_k)
\Big]^{\frac{\alpha - 1}{\alpha}} M^{\frac{1}{\alpha }}$, 
and $\h \cD = \left\{ \sum_{k = 1}^p \lambda_k \h \sD_k\colon \lambda
  \in \Delta \right\}$, with $\h \sD_k(x, y) 
= \frac{\h \sQ(k | x) \sD(x, y)}{\h \sQ(k)}$.
\end{theorem}

\subsection{Optimization Algorithm}
\label{subsec:alg}

By Proposition~\ref{prop:hzgz}, to determine the parameter $z'$
guaranteeing the bound of Theorem~\ref{th:nestimate-bis} for
$\h g_{z'}$, it suffices to determine the parameter $z$ that yields
the guarantee of Theorem~\ref{th:nestimate} for $\h h_z$, when using
the estimates $\h \sD_k = \frac{\h \sQ(k | x) \sD(x)}{\h \sQ(k)}$. As
shown by \cite{hoffman2018algorithms}, the parameter $z$ is the one
for which $\h h_z$ admits the same loss for all source domains, that
is $\cL(\h \sD_k, \h h_z) = \cL(\h \sD_{k'}, \h h_z)$ for all
$k, k' \in [p]$, where $\h \sD_k$ is the joint distribution derived
from $\h \sD_k$:
$\h \sD_k(x, y) = \h \sD_k(x) \sD(y | x) = \frac{\h \sQ(k | x) \sD(x,
  y)}{\h \sQ(k)}$, with
$\sD(x, y) = \frac{1}{p} \sum_{k = 1}^p \sD_k(x, y)$. Note,
$\h \sD (x, y)$ is abusively denoted the same way as $\h \sD(x)$ 
to avoid the introduction of additional notation, 
but the difference in arguments should suffice to help distinguish 
the two distributions.

Thus, using $\h g_{z'} = \h h_z$, to find $z$, and subsequently $z'$, it
suffices to solve the following optimization problem in $z$:
\begin{align}
\label{eq:minmax}
\min_{z \in \Delta}\max_{k \in [p]}  \quad
\cL(\h \sD_k ,\h g_{z'})  - \cL(\h \sD_z, \h g_{z'}),
\end{align}
where $z'_k = \frac{z_k/\h \sQ(k)}{\sum_{j = 1}^p z_j/\h \sQ(j)}$ and
$\h \sD_z = \sum_{k = 1}^p z_k \h \sD_k$.  As in previous work, this
problem can be cast as a DC-programming (difference-of-convex) problem
and solved using the DC algorithm
\citep{TaoAn1997,TaoAn1998,SriperumbudurLanckriet2009}.  However, we
need to derive a new DC-decomposition here, both for the regression
and the probability model, since the objective is distinct from that
of previous work. A detailed description of that DC-decomposition and
its proofs, as well as other details of the algorithm are given in
the full version of the paper. 

\section{Learning Guarantees}
\label{sec:theory}

In this section, we prove favorable learning guarantees for the
predictor $\h g_z$ returned by \dmsa, when using conditional maximum
entropy to derive domain estimates $\sQ(k | x)$. We first extend
Theorem~\ref{th:nestimate} and present a general theoretical guarantee
which holds for \dmsa\ and \gmsa\ (Section~\ref{sec:general}). Next,
in Section~\ref{sec:maxent}, we give a generalization bound for
conditional Maxent and use that to prove learning guarantees for
\dmsa. We then analyze \gmsa\ using kernel density estimation
(Section~\ref{sec:compare}), and show that \dmsa\ benefits from
significantly more favorable learning guarantees than \gmsa.

\subsection{General Guarantee}
\label{sec:general}

Theorem~\ref{th:nestimate} gives a guarantee in terms of a R\'enyi
divergence of $\sD_T$ and $\h \cD$, which depends on the empirical
estimates. Instead, we derive a bound in terms of a R\'enyi divergence
of $\sD_T$ and $\cD$  and, as with 
Theorem~\ref{th:nestimate}, the R\'enyi divergences between the
distributions $\sD_k$ and their estimates $\h \sD_k$.

To do so, we first prove an inequality that can be viewed as 
a triangle inequality
result for R\'enyi divergences.

\begin{restatable}{proposition}{TriangleInequalityProposition}
\label{prop:tri_ineq}
Let $\sP$, $\sQ$, $\sR$ be three distributions on $\sX \times
\sY$. Then, for any $\gamma \in (0, 1)$ and any $\alpha > \gamma$, the
following inequality holds:
\begin{align*}
\Big[\sfd_\alpha(\sP\parallel \sQ) \Big]^{\alpha-1}
\leq \Big[\sfd_{\frac{\alpha}{\gamma}}(\sP \parallel \sR)
\Big]^{\alpha-\gamma}
\Big[\sfd_{\frac{\alpha-\gamma}{1-\gamma}}(\sR \parallel \sQ)
\Big]^{\alpha-1}.
\end{align*}
\end{restatable}
This result
is used in combination with Theorem~\ref{th:nestimate}
to establish the following.

\begin{restatable}{theorem}{GuaranteeTheorem}
\label{th:guarantee}
For any $\delta > 0$, there exists $z \in \Delta$ such
that the following inequality holds for any $\alpha > 1$
and arbitrary target distribution $\sD_T$:
\begin{align*}
\cL(\sD_T, \h g_z) 
\leq [(\h \e + \delta) \, \h \sfd']^{\frac{\alpha-1}{\alpha}} 
[\sfd_{{2\alpha}}(\sD_T \parallel \cD)]^{\frac{2\alpha-1}{2\alpha}} M^{\frac{1}{\alpha}},
\end{align*}
where 
$\h \e =  (\e  \h \sfd)^{\frac{\alpha - 1}{\alpha}}
M^{\frac{1}{\alpha }}$, 
$\h \sfd = \max_{k\in[p]} 
\sfd_\alpha(\h \sD_k \parallel \sD_k)$, and
$\h \sfd' = \max_{k\in[p]} \sfd_{{2\alpha-1}}
(\sD_{k} \parallel \h \sD_{k})$, 
with $\h \sD_k = \frac{\h \sQ(k | x) \sD(x)}{\h \sQ(k)}$.
\end{restatable}
 The theorem holds
similarly for \gmsa\ with $\h \sD_k$ a direct estimate of $\sD_k$.  This
provides a strong performance guarantee for \gmsa\ or \dmsa\ when the
target distribution $\sD_T$ is close to the family of mixtures of the
source distributions $\sD_k$, and when $\h \sD_k$ is a good estimate
of $\sD_k$.

\subsection{Conditional Maxent}
\label{sec:maxent}

The distribution $\sD = \frac{1}{p} \sum_{k = 1}^p \sD_k$ over
$\sX \times \sY$ naturally induces the distribution $\sQ$ over
$\sX \times [p]$ defined for all $(x, k)$ by
$\sQ(x, k) = \frac{1}{p} D_k(x)$. Let
$S = ((x_1, k_1), \ldots, (x_m, k_m))$ be a sample of $m$ 
labeled points drawn i.i.d.\! from $\sQ$. 

Let $\Phi\colon \sX \times [p] \to \Rset^N$ be a feature
mapping with bounded norm, $\| \Phi \| \leq r$, for some
$r > 0$.
Then, the
optimization problem defining the solution of conditional
Maxent (or multinomial logistic regression) with the feature
mapping $\Phi$ is given by
\begin{equation}
\label{eq:LogisticRegression}
\min_{w \in \Rset^N} \mu \| w \|^2 
- \frac{1}{m} \sum_{i = 1}^m \log \sfp_w [k_i |x_i],
\end{equation}
where $\sfp_w$ is defined by
$\sfp_w[k | x] = \frac{1}{Z(x)} \exp(w \cdot \Phi(x, k))$, with
$Z(x) = \sum_{k \in [p]} \exp(w \cdot \Phi(x, k))$,
and where $\mu \geq 0$ is a regularization parameter.
Then, conditional Maxent benefits from the
following theoretical guarantee.

\begin{restatable}{theorem}{LogisticRegressionTheorem}
\label{th:LogisticRegression}
Let $\h w$ be the solution of problem \eqref{eq:LogisticRegression} 
and $w^*$ the population solution of the conditional
Maxent optimization problem:
\begin{align*}
  w^* = \argmin_{w\in\Rset^N} \, \mu \lVert w \rVert^2 
  - \E_{(x, k) \sim \sQ} \bracket[\big]{\log \sfp_w [k | x]}.
\end{align*}
Then, for any $\delta > 0$, with probability at least $1 - \delta$, 
for any $(x, k) \in \sX \times [p]$, the following inequality holds:
\begin{equation*}
\abs[\Big]{ \log \sfp_{\h w}[k | x] - \log \sfp_{w^*}[k | x] }
\leq \frac{2\sqrt{2} \mspace{1mu} r^2}{\mu\sqrt{m}} 
\bracket*{1 + \sqrt{\log (1/\delta)}}.
\end{equation*}
\end{restatable}
The theorem shows that the pointwise log-loss of the
conditional Maxent solution $\sfp_{\h w}$ is close to that of the
best-in-class $\sfp_{w^*}$ modulo a term in
$O(1/\sqrt{m})$ that does not depend on the dimension 
of the feature space.

\subsection{Comparison of the Guarantees for \dmsa\ and \gmsa}
\label{sec:compare}

We now use Theorem~\ref{th:guarantee} and
the bound of Theorem~\ref{th:LogisticRegression} to give a theoretical guarantee
for \dmsa\ used with conditional Maxent. We  show that it is more favorable than a guarantee
for \gmsa\ using kernel density estimation.

\begin{restatable}[\dmsa]{theorem}{GZTheorem}
\label{th:gz}
There exists $z \in \Delta$ such that for any $\delta > 0$, with
probability at least $1 - \delta$ the following inequality holds
\dmsa\ used with conditional Maxent, for an arbitrary target mixture
$\sD_T$:
\begin{align*}
\cL(\sD_T, \h g_z)
%& \leq \e \, p \, e^{6r \| \h w - w^* \|} \, \sfd^* \, \sfd'^*\\
& \leq \e \, p \, e^{\frac{6\sqrt{2}r^2}{\mu\sqrt{m}} 
\bracket*{ 1 + \sqrt{\log (1/\delta)} }} 
\, \sfd^* \, \sfd'^*,\\[.2cm]
\text{with} \quad \sfd^* & = \sup_{x \in \sX} \sfd_\infty
\left(\sQ^*[\cdot | x] \parallel \sQ(\cdot | x) \right),\\
\sfd'^* & = \sup_{x \in \sX} \sfd^2_\infty
\left( \sQ(\cdot | x) \parallel \sQ^*[\cdot | x] \right),
\end{align*}
where $\sQ^*(\cdot | x) = \sfp_{w^*}[\cdot | x]$
is the population solution of 
conditional Maxent problem (statement
of Theorem~\ref{th:LogisticRegression}).
\end{restatable}
 The
theorem shows that the expected error of \dmsa\ with conditional
Maxent is close to $\e$ modulo a factor that varies as
$e^{1/\sqrt{m}}$, where $m$ is the size of the total unlabeled sample
received from all $p$ sources, and factors $\sQ^*$ and $\sQ'^*$ that
measure how closely conditional Maxent can approximate the true
conditional probabilities with infinite samples.

Next, we prove learning guarantees for \gmsa\ with densities estimated
via kernel density estimation (KDE). We assume that the same i.i.d.\
sample $S = ((x_1, k_1), \ldots, (x_m, k_m))$ as with conditional
Maxent is used. Here, the points labeled with $k$ are used for
estimating $\sD_k$ via KDE. Since the sample is drawn from $\sQ$ with
$\sQ(x, k) = \frac{1}{p} \sD_k$, the number of samples points $m_k$
labeled with $k$ is very close to $\frac{m}{p}$.  $\h \sD_k$ is learned
from $m_k$ samples, via KDE with a normalized kernel function
$K_\sigma(\cdot, \cdot)$ that satisfies
$\int_{x\in\sX} K_\sigma(x, x') \, dx = 1$ for all $x' \in \sX$.

\begin{restatable}[\texttt{GMSA}]{theorem}{SimpleEstimateHZTheorem}
\label{th:simple_estimate_hz}
There exists $z \in \Delta$ such that, for any $\delta > 0$, with
probability at least $1 - \delta$ the following inequality holds for
\gmsa\ used KDE, for an arbitrary target mixture $\sD_T$:
\begin{align*}
\cL(\sD_T, \h h_z) \leq
\e^{\frac{1}{4}} M^{\frac{3}{4}}
e^{\frac{6 \kappa}{\sqrt{2(m/p)}} \sqrt{\log p + \log (1/\delta)}}
\sfd^*
\sfd'^*,
\end{align*}
with $\kappa = \max_{x, x', x'' \in \sX} 
\frac{K_\sigma(x, x')}{K_\sigma(x, x'')}$, and 
\begin{align*}
\sfd^* & = \max_{k \in [p]} \E_{x \sim \sD_k} [
\sfd_{+\infty}\big(K_\sigma(\cdot, x) \parallel \sD_k\big) ], \\
\sfd'^* & = \max_{k \in [p]} \E_{x \sim \sD_k} [
\sfd_{+\infty}\big(\sD_k \parallel K_\sigma(\cdot, x) \big)].
\end{align*}
\end{restatable}
In comparison with the guarantee for \dmsa, the bound for \gmsa\
admits a worse dependency on $\e$. Furthermore, while the dependency
of the learning bound of \dmsa\ on the sample size is of the form
$O(e^{1/\sqrt{m}})$ and thus decreases as a function of the full
sample size $m$, that of \gmsa\ is of the form $O(e^{1/\sqrt{m/p}})$
and only decreases as a function of the per-domain sample size. This
further reflects the benefit of our discriminative solution since the
estimation of the conditional probabilities is based on conditional
Maxent trained on the full sample. 
Finally, the bound of \gmsa\ depends on $\kappa$, 
a ratio that can be unbounded for Gaussian kernels commonly used for KDE.

The generalization guarantees for \dmsa\ depends on two critical terms that
measure the divergence between the population solution of conditional
Maxent and the true domain classifier $\sQ(\cdot | x)$:
\begin{align*}
\sfd_{+\infty} \big(\sQ^*(\cdot | x) \parallel \sQ(\cdot | x) \big) 
\quad \text{and} \quad
\sfd_{+\infty} \big(\sQ(\cdot | x) \parallel \sQ^*(\cdot | x)\big).
\end{align*}
When the feature mapping for conditional Maxent is sufficiently
rich, for example when it is the reproducing kernel Hilbert space (RKHS)
associated to a Gaussian kernel, one can expect the two divergences to
be close to one. The generalization guarantees for
\gmsa\  also depend on two
divergence terms:
\begin{align*}
\sfd_{+\infty} \big(K_\sigma(\cdot, x)
\parallel \sD_k\big)
\quad \text{and} \quad
\sfd_{+\infty}\big(\sD_k \parallel K_\sigma(\cdot, x)\big).
\end{align*}
Compared to learning a domain classifier 
$\h \sQ(\cdot | x)$, it is more difficult to chose a good density kernel
$K_\sigma(\cdot, \cdot)$ to ensure that the divergence between marginal
distributions is small, which shows another benefit of
\dmsa.

The next section shows that, in addition to these theoretical
advantages, \dmsa\ also benefits from more favorable empirical 
results.

\begin{table*}[t]
\caption{MSE on the sentiment analysis dataset. Single source
baselines, \texttt{\small K}, \texttt{\small D}, \texttt{\small B}, \texttt{\small E}, 
the uniform combination
\unif, \gmsa, and \dmsa.} 
\vskip .0in
  \resizebox{1\linewidth}{!}{
  \begin{tabular}{l ccccccccccc}
  \toprule
  & \multicolumn{11}{c}{Sentiment Analysis Test Data}\\
  \cline{2-12}
  & \texttt{\small K}& \texttt{\small D} & \texttt{\small B}	& \texttt{\small E}	&\texttt{\small KD}	&\texttt{\small BE}	& \texttt{\small DBE} & \texttt{\small KBE} & \texttt{\small KDB} &	\texttt{\small KDB} &	\texttt{\small KDBE} \\
  \midrule

\texttt{K}  &   {1.42$\pm$0.10}  &   {2.20$\pm$0.15} &   {2.35$\pm$0.16} &   {1.67$\pm$0.12} &   {1.81$\pm$0.07} &   {2.01$\pm$0.10} &   {2.07$\pm$0.08}  &   {1.81$\pm$0.06} &   {1.76$\pm$0.06} &   {1.99$\pm$0.06} &   {1.91$\pm$0.05}\\
\texttt{D}  &   {2.09$\pm$0.08} &   {1.77$\pm$0.08} &   {2.13$\pm$0.10} &   {2.10$\pm$0.08} &   {1.93$\pm$0.07} &   {2.11$\pm$0.07} &   {2.00$\pm$0.06}  &   {2.11$\pm$0.06} &   {1.99$\pm$0.06} &   {2.00$\pm$0.06} &   {2.02$\pm$0.05}\\
\texttt{B}  &   {2.16$\pm$0.13} &   {1.98$\pm$0.10} &   {1.71$\pm$0.12} &   {2.21$\pm$0.07} &   {2.07$\pm$0.11} &   {1.96$\pm$0.07} &   {1.97$\pm$0.06}  &   {2.03$\pm$0.06} &   {2.12$\pm$0.07} &   {1.95$\pm$0.08} &   {2.02$\pm$0.06}\\
\texttt{E}  &   {1.65$\pm$0.09} &   {2.35$\pm$0.11} &   {2.45$\pm$0.14} &   {1.50$\pm$0.07} &   {2.00$\pm$0.09} &   {1.97$\pm$0.09} &   {2.10$\pm$0.08}  &   {1.86$\pm$0.05} &   {1.83$\pm$0.07} &   {2.15$\pm$0.07} &   {1.99$\pm$0.06}\\
\texttt{unif}   &   {1.50$\pm$0.06} &   {\bf1.75$\pm$0.09}  &   {1.79$\pm$0.10} &   {1.53$\pm$0.07} &   {1.63$\pm$0.06} &   {1.66$\pm$0.08} &   {1.69$\pm$0.06}  &   {1.61$\pm$0.05} &   {1.60$\pm$0.05} &   {1.68$\pm$0.05} &   {1.65$\pm$0.05}\\
\texttt{\gmsa} &   {1.42$\pm$0.10} &   {1.88$\pm$0.11} &   {1.80$\pm$0.10} &   {1.51$\pm$0.07} &   {1.65$\pm$0.08} &   {1.66$\pm$0.07}  &   {1.73$\pm$0.05}  &   {1.58$\pm$0.04}  &   {1.60$\pm$0.05} &   {1.70$\pm$0.04} &   {1.65$\pm$0.04}\\
  \texttt{\dmsa} (ours) &   {\bf1.42$\pm$0.08} &   {1.76$\pm$0.07} &   {\bf1.70$\pm$0.11}  &   {\bf1.46$\pm$0.07}  &   {\bf1.59$\pm$0.06}  &   {\bf1.58$\pm$0.07}   &   {\bf1.64$\pm$0.05}  &   {\bf1.53$\pm$0.04}  &   {\bf1.55$\pm$0.04}  &   {\bf1.63$\pm$0.04}  &   {\bf1.59$\pm$0.04}\\
\bottomrule
\end{tabular}
}
\label{table:sa}
\vskip -.15in
\end{table*}

\section{Experiments}
\label{sec:experiments}

We evaluated our \dmsa\ technique on the same
datasets as those used in \citep{hoffman2018algorithms}, as well
as with the UCI adult dataset.   Since
\cite{hoffman2018algorithms} has shown that \gmsa\ empirically
outperforms various alternative MSA solutions, in this section, we
mainly focus on demonstrating improvements over \gmsa\ under the same experimental
setups.

\textbf{Sentiment analysis.} 
To evaluate the \dmsa\ solution under the regression model, we used
the sentiment analysis dataset \citep{blitzer_acl07}, which consists
of product review text and rating labels taken from four domains:
\texttt{\small books} (B), \texttt{\small dvd} (D), \texttt{\small
  electronics} (E), and $\texttt{\small kitchen}$ (K), with
$2\mathord,000$ samples for each domain.
We adopted the same training procedure and hyper-parameters as those
used by \cite{hoffman2018algorithms} to obtain base predictors: first
define a vocabulary of $2\mathord,500$ words that occur at least twice
in each of the four domains, then use this vocabulary to define
word-count feature vectors for every review text, and finally train
base predictors for each domain using support vector regression.  We
use the same word-count features to train the domain classifier via
logistic regression.  We randomly split the $2\mathord,000$ samples
per domain into
$1\mathord,600$ train and $400$ test samples for each domain, and
learn the base predictors, domain classifier, density estimations, and
parameter $z$ for both MSA solutions on all available training samples.
We repeated the process 10 times, and report the mean and standard
deviation of the mean squared error on various target test mixtures in
Table~\ref{table:sa}.

We compared our technique, \dmsa, against each source predictor,
$h_k$, the uniform combination of the source predictors (\texttt{\small
  unif}), $\frac{1}{p}\sum_{k = 1}^p h_k$, and \gmsa\
with kernel density estimation.  Each column in Table~\ref{table:sa}
corresponds to a different target test mixture, as indicated by the
column name: four single domains, and uniform mixtures of two, three,
and four domains, respectively.  Our distribution-weighted method
\dmsa\ outperforms all baseline predictors across
almost all test domains.  {Observe that, even when the target is a
  single source domain, such as \texttt{K}, \texttt{B}, \texttt{E},
  our method can still outperform the predictor which is trained and
  tested on the same domain, showing the benefits of ensembles.}
Moreover, \dmsa\ improves upon \texttt{\small
  \gmsa} by a wide margin on all test mixtures, which demonstrates the
advantage of using a domain classifier over estimated densities in the
distribution-weighted combination.

\textbf{Recognition tasks with the cross-entropy loss.}
To evaluate the \dmsa\ solution under the probability model, we
considered a digit recognition task consists of three datasets: Google
Street View House Numbers (SVHN), MNIST, and USPS. 
For each individual domain, we trained a convolutional neural network
(CNN) with the same setup as in \cite{hoffman2018algorithms}, and used
the output from the softmax score layer as our base predictors
$h_k$. Furthermore, for every input image, we extracted the last layer
before softmax from each of the base networks and concatenated them to
obtain the feature vector for training the domain classifier. We used
the full training sets per domain to train the source model, and used
$6\mathord,000$ samples per domain to learn the domain classifier.
Finally, for our DC-programming algorithm, we used a $1\mathord,000$
image-label pairs from each domain, thus a total of $3\mathord,000$
labeled pairs to learn the parameter $z$.

We compared our method \dmsa\ against each source predictor ($h_k$),
the uniform combination, \unif,
a network jointly trained on all source data combined, \joint, and \gmsa\ with kernel density
estimation. Since the training and testing datasets are fixed, we
simply report the numbers from the original \gmsa\ paper. We evaluated
these baselines on each of the three test datasets, on combinations of
two test datasets, and on all test datasets combined. The results are
reported in Table~\ref{table:digits}.  Once again, \dmsa\ outperforms
all baselines on all test mixtures, and when the target is a single
test domain, \dmsa\ admits a comparable performance to the predictor
that is trained and tested on the same domain. And, as in the
sentiment analysis experiments, \dmsa\ outperforms
\gmsa\ by a wide margin on all test domains.

\begin{figure}[t]
    \centering
    \includegraphics[scale=0.4]{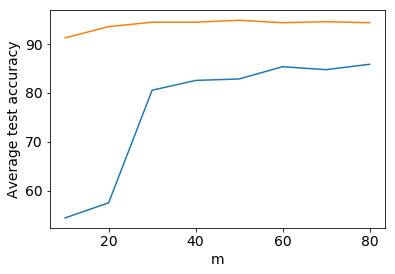}
    \vskip -.2in
    \caption{Average test accuracy of \gmsa\ (blue) and \dmsa\
      (orange) on the digits dataset as a function of the number of
      samples used in domain adaptation. }
    \label{fig:digits_num_samples}
\end{figure}

To illustrate the efficiency of \dmsa\, we further evaluated \dmsa\
and \gmsa\ on the digits dataset when only a small amount of data is available for domain
adaptation. We varied $m$, the number of samples per domain and
evaluated both \dmsa\ and \gmsa, see
Figure~\ref{fig:digits_num_samples}. As expected, %especially in the small sample regime,
\dmsa\ consistently outperforms \gmsa, thus matching our theoretical analysis that \dmsa\ can succeed with fewer samples. 

\begin{table}[t]
\centering
\caption{Digit Dataset Accuracy. 
\dmsa\ outperforms each single-source domain model, \unif, \joint, and most importantly \gmsa, on various target mixtures.  }
\vskip 0in
%\resizebox{1\linewidth}{!}{%
\begin{tabular}{l cccccccg}
\toprule
& \multicolumn{8}{c}{Digits Test Data}\\
\cline{2-8}
& \texttt{svhn}& \texttt{mnist}  & \texttt{usps} & \texttt{mu} & \texttt{su} & \texttt{sm} & \texttt{smu} & mean\\
\midrule
CNN-\texttt{s}  &  \bf {92.3} &   {66.9} &   {65.6} &   {66.7} &   {90.4} &   {85.2} &   {84.2} &   {78.8}\\
CNN-\texttt{m}  &   {15.7} &  \bf {99.2} &   {79.7} &   {96.0} &   {20.3} &   {38.9} &   {41.0} &   {55.8}\\
CNN-\texttt{u}  &   {16.7} &   {62.3} &  \bf {96.6} &   {68.1} &   {22.5} &   {29.4} &   {32.9} &   {46.9}\\
CNN-\unif   &   {75.7} &   {91.3} &   {92.2} &   {91.4} &   {76.9} &   {80.0} &   {80.7} &   {84.0}\\
CNN-\joint&   {90.9} &   {99.1} &   {96.0} &   {98.6} &   {91.3} &   {93.2} &   {93.3} &   {94.6}\\
\gmsa   &   {91.4} &   {98.8} &   {95.6} &   {98.3} &   {91.7} &   {93.5} &   {93.6} &   {94.7}\\
\dmsa\ (ours)   &   {92.3} &   {99.2} &   {96.6} &   \bf {98.8} &   \bf{92.6} &   \bf{94.2} &   \bf{94.3} &   \bf{95.4}\\
\bottomrule
\end{tabular}
%}
\label{table:digits}
\end{table}

\textbf{Adult dataset}. 
We also experimented with the UCI adult dataset \citep{blake1998uci}. It contains $32,561$ training samples with numerical and categorical features, each representing a person.
The task consists of predicting if the person’s income exceeds $50,000$ dollars.
Following \citep{mohri2019agnostic}, we split the dataset into two domains, the doctorate, \texttt{Doc}, domain and non-doctorate, \texttt{NDoc}, domain and used categorical features for training linear classification models. We froze these models and experimented with domain adaptation. Here, we repeatedly sampled $400$ training samples from each domain for training, keeping the test set fixed. 

The results are in Table~\ref{tab:adult}. \dmsa\ achieves higher accuracy compared to \gmsa\ 
on the \texttt{NDoc} domain and also in the average of two domains. The difference in performance is not
statistically significant for the \texttt{Doc} domain as it has very few test samples.
\vspace{-2ex}
\begin{table}[h!]
\centering
\caption{Linear models for adult dataset. The experiments are averaged over $100$ runs.} 
%\resizebox{1\linewidth}{!}{%
\begin{tabular}{l ccg }
\toprule
Test data & \texttt{\small Doc}  & \texttt{\small NDoc} &  \texttt{\small Doc-NDoc}
 \\
\midrule
\gmsa\   &70.2 $\pm$ 1.2  &76.4 $\pm$ 1.6 & 73.3 $\pm$ 0.8\\   
\dmsa\   &70.0 $\pm$ 0.8  &\bf{80.5 $\pm$ 0.5} & \bf{75.3 $\pm$ 0.4}\\ 
\bottomrule
\end{tabular}
%}
\label{tab:adult}
\end{table}

We finally conducted simulations on a small synthetic dataset 
to illustrate the difference between \gmsa\
and \dmsa. We used the \texttt{\small sklearn} toolkit for these
experiments. Let $\sD_1$ and $\sD_2$ be Gaussian mixtures in one dimensions as 
follows:
$\sD_1 = 0.9 \cdot N(-20, 8) + 0.1  \cdot N(0, 0.1)$ and $\sD_2 = 0.75  \cdot N(3, 0.1) + 0.25 \cdot N(5, 0.1) + 0.05  \cdot N(0, 0.1)$, see Figure~\ref{fig:illustration}.
The two domains are similar around $0$ but are disjoint otherwise. Let the labeling function $f(x) = -1 \text{ if } x \in [-0.5, 0.5] \cup [3.5, \infty)$. The example is designed such that if their estimates are good, then both \gmsa\ and \dmsa\ would 
%if the true domain distributions are known, then \gmsa\ 
achieve close to $100\%$ accuracy.  We first sampled $1000$ examples and trained a linear separator $h_k$ for each domain $k$. % We compared \gmsa\ and \dmsa\ on this dataset. 
For \gmsa, we trained kernel density estimators and chose the bandwidth based on a five-fold cross-validation. For \dmsa, we trained a conditional Maxent threshold classifier. We first illustrate the kernel density estimate using $1000$ samples in Figure~\ref{fig:illustration}. For $x \in [-0.5, 0.5]$, $\sD_1(x) > \sD_2(x)$, but the kernel density estimates satisfy $\h \sD_2(x) \geq \h \sD_1(x)$, which shows the limitations of kernel density estimation with a single bandwidth. On the other hand, \dmsa\ selected a threshold around $0.3$ for distinguishing between $\sD_1$ and $\sD_2$ and achieves accuracy around $100\%$. We varied the number of examples available for domain adaptation and
compared \gmsa\ and \dmsa. For simplicity we found the best $z$
using exhaustive search for both \gmsa\ and \dmsa.  The results show
that \dmsa\ consistently outperforms \gmsa\ on both the domains
and hence on all convex combinations, 
 see Figure~\ref{fig:artificial}. The results also show that \dmsa\ converges quickly in accordance with our theory.

\begin{figure}[t]
    \centering
    \includegraphics[scale=0.4]{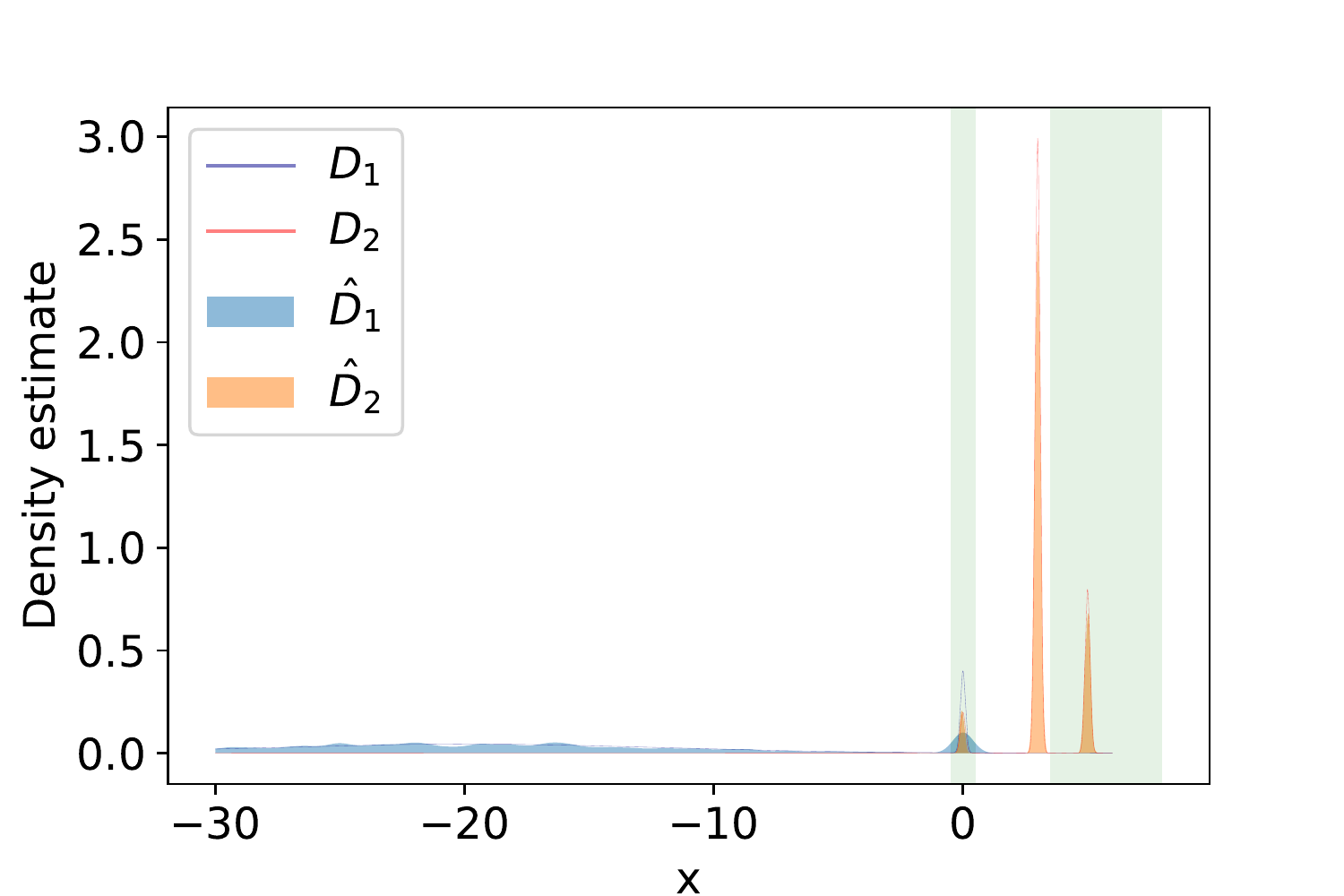}
     \includegraphics[scale=0.4]{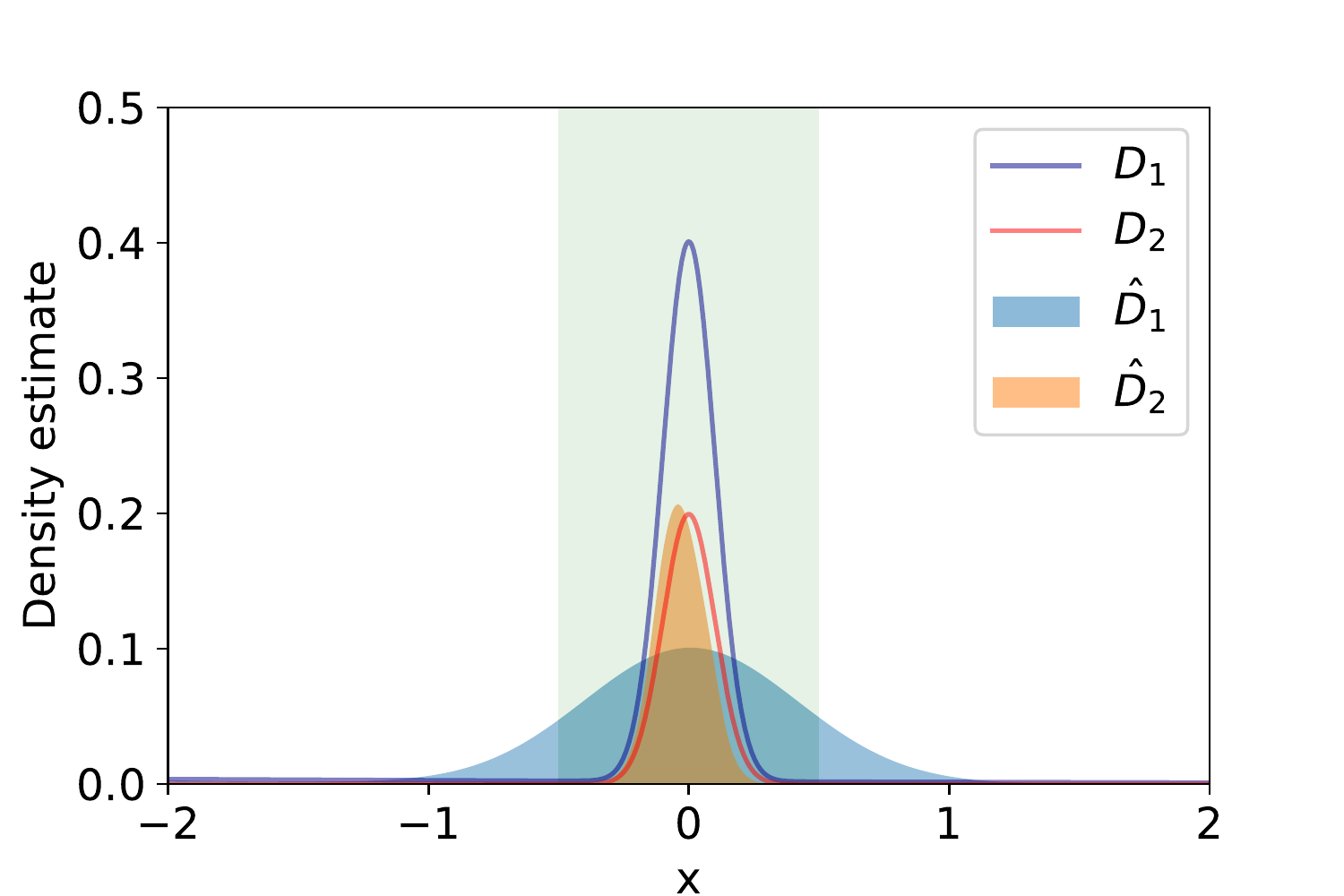}
    \vskip -.2in
    \caption{Left: Illustration of the true densities and kernel density estimates for \gmsa\ for domains $\sD_1$ and $\sD_2$
    with $1000$ samples.  
    The labeling function $f(x) = -1$ in the green regions and 
    $1$ otherwise. Right: Same estimates zoomed in at $x=0$.
    }
    \label{fig:illustration}
    \vskip -.1in
\end{figure}

\begin{figure}[t]
    \centering
    \includegraphics[scale=0.4]{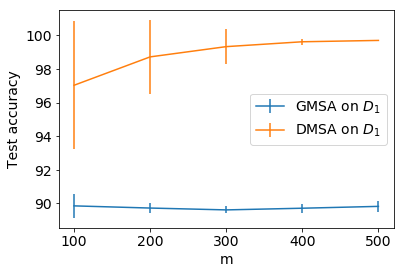}
     \includegraphics[scale=0.4]{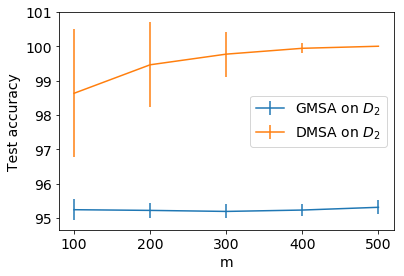}
\vskip -0.2in
    \caption{Comparison of \gmsa\ and \dmsa\ on the synthetic
      dataset. \dmsa\ performs better than \gmsa\ on both the domains
      and hence on any convex combination. The experiments are
      averaged over $10$ runs. The error bars show one standard
      deviation.}
    \label{fig:artificial}
\vskip -.1in
\end{figure}

\vspace{-2ex}
\section{Conclusion}

We presented a new algorithm for the important problem of
multiple-source adaptation, which commonly arises in applications. Our
algorithm was shown to benefit from favorable theoretical guarantees
and a superior empirical performance, compared to previous work.
Moreover, our algorithm is practical: it is straightforward to train a
multi-class classifier in the setting we described and our
DC-programming solution is very efficient. 

Providing a robust solution for the problem is particularly important
for under-represented groups, whose data is not necessarily
well-represented in the classifiers to be combined and trained on
source data. Our solution demonstrates improved performance even in
the cases where the target distribution is not included in the source
distributions. We hope that continued efforts in this area will result
in more equitable treatment of under-represented groups.

%\small
\bibliography{mad}

\end{document}